\documentclass{article}

\usepackage[nonatbib,preprint]{neurips_2020}

\usepackage[utf8]{inputenc} %
\usepackage[T1]{fontenc}    %
\usepackage{hyperref}       %
\usepackage{url}            %
\usepackage{booktabs}       %
\usepackage{amsfonts}       %
\usepackage{nicefrac}       %
\usepackage{microtype}      %
\usepackage{xspace}
\usepackage[pdftex]{graphicx}
\usepackage{xspace}
\usepackage{multirow}
\usepackage{subfigure}
\usepackage{enumitem}
\usepackage{algorithm}
\usepackage{algorithmic}
\usepackage{todonotes}
\usepackage{amssymb} 
\usepackage{wrapfig}

\title{Statistical Bootstrapping for Uncertainty Estimation \\ in Off-Policy Evaluation}

\author{%
  Ilya Kostrikov\thanks{Also at NYU.} \\
  Google Research \\
  \texttt{kostrikov@google.com} \\
  \And
  Ofir Nachum \\
  Google Research \\
  \texttt{ofirnachum@google.com} \\
}

\usepackage{amsmath,amsfonts,amsthm,bm}

\usepackage{color}
\usepackage[normalem]{ulem}

\newcommand{\Bo}[1]{}
\newcommand{\comment}[1]{}

\newtheorem{theorem}{Theorem}

\newtheorem{assumption}{Assumption}
\newtheorem{definition}{Definition}

\def\eqref#1{(\ref{#1})}

\def\1{\bm{1}}

\DeclareMathAlphabet{\mathsfit}{\encodingdefault}{\sfdefault}{m}{sl}
\SetMathAlphabet{\mathsfit}{bold}{\encodingdefault}{\sfdefault}{bx}{n}

\newcommand{\Qhat}{\widehat{Q}}

\newcommand{\E}{\mathbb{E}}

\newcommand{\R}{\mathbb{R}}

\newcommand{\nbr}[1]{\left\|#1\right\|}

\renewcommand{\url}[1]{{\sffamily #1}}

\def\mdp{\mathcal{M}}

\def\visitpi{d^\pi}

\def\visitrb{d^\Dset}

\DeclareMathOperator{\avgstep}{\rho}

\def\Sset{\mathcal{S}}
\def\Aset{\mathcal{A}}
\def\Reward{\mathcal{R}}
\def\Transition{\mathcal{T}}
\def\Dset{\mathcal{D}}
\def\Dnset{\mathcal{D}_n}

\def\init{\mu_0}
\def\defeq{:=}
\def\Qhat{\widehat{Q}}
\renewcommand{\widehat}{\hat}
\def\Rmax{R_{\mathrm{max}}}
\def\Wmax{W_{\mathrm{max}}}
\def\visitDn{d^{\Dnset}}
\def\initDn{{\mu}_0^{\Dnset}}
\def\RewardDn{{\Reward}_{\Dnset}}
\def\TransitionDn{{\Transition}_{\Dnset}}
\def\avgstephat{\avgstep_{\mathrm{DM}}}
\def\Dtilde{\widetilde{\mathcal{D}}}
\def\visitDtilde{d^{\Dtilde}}
\def\visitpitilde{\visitpi_{\Dtilde}}
\def\initDtilde{\mu_0^{\Dtilde}}
\def\RewardDtilde{\Reward_{\Dtilde}}
\def\TransitionDtilde{\Transition_{\Dtilde}}
\def\AvgRewardDtilde{\overline{\Reward}_{\Dtilde}}
\def\Etilde{\widetilde{\mathcal{E}}}
\def\visitEtilde{d^{\Etilde}}
\def\initEtilde{\mu_0^{\Etilde}}

\def\TransitionEtilde{\Transition_{\Etilde}}
\def\AvgRewardEtilde{\overline{\Reward}_{\Etilde}}
\def\Qpitilde{Q^\pi_{\Dtilde}}
\def\RewardBDn{{\Reward}^\kappa_{\Dnset}}
\def\TransitionBDn{{\Transition}^\kappa_{\Dnset}}
\def\RewardBDtilde{{\Reward}^\kappa_{\Dtilde}}
\def\TransitionBDtilde{{\Transition}^\kappa_{\Dtilde}}

\def\Rewardprior{\Reward_{\mathrm{prior}}}
\def\Transitionprior{\Transition_{\mathrm{prior}}}
\def\avgstepbiased{\avgstephat^\kappa}
\def\rnoise{R_{\mathrm{noise}}}

\begin{document}

\maketitle

\begin{abstract}
In reinforcement learning, it is typical to use the empirically observed transitions and rewards to estimate the value of a policy via either model-based or $Q$-fitting approaches.
Although straightforward, these techniques in general yield biased estimates of the true value of the policy.
In this work, we investigate the potential for statistical bootstrapping to be used as a way to take these biased estimates and produce calibrated confidence intervals for the true value of the policy.
We identify conditions 
-- specifically, sufficient data size and sufficient coverage -- 
under which statistical bootstrapping in this setting is guaranteed to yield correct confidence intervals.
In practical situations, these conditions often do not hold, and so we discuss and propose mechanisms that can be employed to mitigate their effects.
We evaluate our proposed method and show that it can yield accurate confidence intervals in a variety of conditions, including challenging continuous control environments and small data regimes.  
\end{abstract}
\section{Introduction}

Providing accurate and trustworthy estimates of a policy's long term value in a decision-making environment is an important problem in reinforcement learning (RL).
Typically, due to cost or safety constraints, one must perform this estimation without actually running the policy in the live environment. Instead, one must predict the value of the policy using only a limited set of experience of some other logging (or behavior) policies acting in the sequential environment. 
This problem is generally referred to as \emph{off-policy evaluation} (OPE)~\cite{precup2000eligibility}.
The OPE problem is especially relevant to many practical domains, such as health~\cite{murphy2001marginal,liao2019off}, education~\cite{mandel2014offline}, and recommendation systems~\cite{Swaminathan17OP}, where accurate evaluation of a new policy is critical to maximize safety and minimize risks associated with deployment of a new policy~\cite{thomas2015safe}.

Perhaps the most straightforward approach to OPE is to use the given finite dataset of experience to determine the environment's empirically observed initialization, transition, and reward probabilities, and then to evaluate the expected value of the target policy in this \emph{empirical} environment. 
This straightforward approach is known as the \emph{direct method} (DM)~\cite{dudik2011doubly,voloshin2019empirical}.
In addition to encompassing \emph{model-based} (MB) methods~\cite{Thomas16DE,hanna2017bootstrapping}, this general paradigm is also implicitly implemented by \emph{$Q$-evaluation} (QE), or its parameteric counterpart \emph{fitted $Q$-evaluation} (FQE)~\cite{paine2020hyperparameter,voloshin2019empirical,bradtke1996linear}. Indeed, the mathematical equivalence of QE and MB, even under certain function approximation schemes, has been recently demonstrated~\cite{duan2020minimax}.

Although the DM paradigm is a straightforward and intuitive approach, it is traditionally seen as undesirable due to it yielding \emph{biased} estimates. That is, the estimates returned by QE or MB over multiple experiments on randomly sampled finite datasets \emph{are not} centered around the true value of the target policy.
This fact has led much of the OPE literature to focus on a variety of importance sampling (IS) based approaches~\cite{precup2000eligibility,li2011unbiased,jiang2015doubly,liu2018breaking,nachum2019dualdice}, for which unbiased estimates are feasible.
However, the ability to provide unbiased estimates is not necessary in many practical applications. Rather, in many practical scenarios where \emph{safety} is a key concern~\cite{thomas2015safe}, the ability to provide unbiased estimates is less relevant than the need for \emph{high-confidence} and accurate lower or upper bounds on the true value of the target policy.
Efron's bootstrap~\cite{efron1987better} is a well-known method in statistics for deriving confidence intervals from biased estimates, and so it may be a promising technique for use in conjunction with DM~\cite{hanna2017bootstrapping}. 
Still, while bootstrapping is a simple approach widely used in statistics, it is not always guaranteed to yield accurate confidence intervals~\cite{putter2012resampling, abadie2008failure}, and in the case of MB or QE, where the OPE estimate is a complex function of the input data, it is not immediately clear whether Efron's bootstrap would be valid.

In this paper, we investigate the validity of Efron's bootstrap applied to DM. We derive theoretical guarantees that show that, if certain conditions are satisfied, Efron's bootstrap applied to DM yields asymptotically accurate confidence intervals. 
The conditions we identify -- namely, sufficient sample size and sufficient coverage of the underlying experience data distribution -- may not hold in many practical scenarios.
Therefore, we use insights from our derivations to suggest mechanisms -- noisy rewards and regularization -- for mitigating the effect of these in practice.
We present empirical results in tabular settings that show the validity of our theory and the benefit of our heuristic mechanisms. 
Extending our methods to more complex environments with function approximation, we present state-of-the-art results, showing that MB and QE with Efron's bootstrap can yield accurate and useful confidence intervals on challenging continuous control benchmarks.
\section{Background}

We consider the standard Markov Decision Process (MDP) setting~\cite{puterman1994markov}, in which the environment is specified by a tuple $\mdp = \langle \Sset, \Aset, \Reward, \Transition, \init, \gamma \rangle$, consisting of a state space $\Sset$, an action space $\Aset$, a reward distribution function $\Reward$, a transition probability function $\Transition$, an initial state distribution $\init$, and a discount $0\le \gamma<1$.  A policy $\pi$ interacts with the environment iteratively, starting with an initial state $s_0 \sim \init$.
For simplicity, we will restrict the text to consider the infinite-horizon setting, although all results apply in the finite horizon setting as well.

In this work, we largely focus on estimation of the \emph{value} of a given target policy $\pi$, defined as the expected accumulated reward of $\pi$ in $\mdp$, averaged over time via $\gamma$-discounting:
\begin{equation}
    \avgstep(\pi) \defeq (1-\gamma)\cdot\E\left[\left.\sum_{t=0}^\infty \gamma^t\cdot r_t ~\right|~ s_0\sim\init,a_t\sim\pi(s_t),r_t\sim\Reward(s_t,a_t),s_{t+1}\sim\Transition(s_t,a_t)\right].
\end{equation}

We consider the \emph{off-policy} setting, in which we do not have explicit knowledge of $\Reward,\Transition,\init$. Rather, we only have access to a finite empirical dataset of experience samples from these distributions.
More concretely, we have a dataset $\Dnset\defeq\{(s_0^{(j)}, s^{(j)}, a^{(j)}, r^{(j)}, s^{\prime(j)})\}_{j=1}^n$ consisting of $n$ tuples $(s_0, s, a, r, s')$ independently sampled via
\begin{equation}
    \label{eq:def-d}
    s_0\sim\init~;~~~ (s, a)\sim\visitrb~;~~~ r\sim \Reward(s,a)~;~~~ s'\sim\Transition(s,a),
\end{equation}
where $\visitrb$ is some unknown distribution over state-action pairs. We will abuse notation at times and use $\visitrb(s_0,s,a,r,s')$ to denote the joint distribution on tuples and $\visitrb(s_0),\visitrb(r,s'|s,a)$ the appropriately marginalized and conditioned distributions. 
The finite dataset $\Dnset$ induces its own empirical distribution over tuples, which we denote 
\begin{equation}
    \visitDn\defeq \frac{1}{n}\sum_{j=1}^n \delta_{(s_0^{(j)}, s^{(j)}, a^{(j)}, r^{(j)}, s^{\prime(j)})},
\end{equation}
where $\delta_\chi$ is the Dirac delta distribution centered at $\chi$.
The empirical distribution over tuples $\visitDn$ in turn determines an empirical initial state distribution $\initDn(s_0)\defeq\visitDn(s_0)$, an empirical reward distribution function $\RewardDn(r|s,a)\defeq\visitDn(r|s,a)$, and an empirical transition probability function $\TransitionDn(s'|s,a)\defeq\visitDn(s'|s,a)$.
To appropriately define $\RewardDn,\TransitionDn$ when $\visitDn$ has poor coverage of the state or action space, we define $\RewardDn(r|s, a)\defeq \Rewardprior(r|s,a)$, $\TransitionDn(s'|s,a)\defeq \Transitionprior(s'|s,a)$ for all $s,a$ such that $\visitDn(s, a)=0$, for some fixed \emph{prior} distribution functions $\Rewardprior, \Transitionprior$. 

The \emph{direct method} (DM) uses the empirically observed $\initDn,\RewardDn,\TransitionDn$ to estimate $\avgstep(\pi)$ as
\begin{equation*}
\hspace{-2mm}
\avgstephat(\pi|\Dnset) \defeq (1-\gamma)\cdot\E\left[\left.\sum_{t=0}^\infty \gamma^t\cdot r_t ~\right|~ s_0\sim\initDn,a_t\sim\pi(s_t),r_t\sim\RewardDn(s_t,a_t),s_{t+1}\sim\TransitionDn(s_t,a_t)\right]. 
\end{equation*}
The direct method may be implemented explicitly through a \emph{model-based} (MB) procedure, where $\initDn,\RewardDn,\TransitionDn$ are either determined analytically or approximated by parameteric models via maximum likelihood. Then, $\avgstephat(\pi|\Dnset)$ is approximated by Monte Carlo trajectories of $\pi$ rolled out using these models.
Alternatively, DM can also be implemented in a model-free fashion via
\emph{$Q$-evaluation} (QE). In this approach, a $Q$-value function $Q:\Sset\times\Aset\to\R$ is iteratively learned via the Bellman backup procedure,
\vspace{-3mm}
\begin{equation}
    \label{eq:qe-backup}
    Q^{(i+1)}(s,a) \leftarrow \E_{\visitDn(r,s'|s,a),a'\sim\pi(s')}\left[r + \gamma Q^{(i)}(s',a')\right].
\end{equation}
Ignoring issues of function approximation, this procedure converges to a fixed point $\Qhat^\pi=\lim_{i\to\infty} Q^{(i)}$, which is the $Q$-value function of $\pi$ under the empirical MDP.\footnote{When $\visitDn$ has poor coverage, the fixed point $\hat{Q}^\pi$ depends on the initial $Q$-values $Q^{(0)}$. The fixed point $\hat{Q}^\pi$ is still the $Q$-value function of $\pi$ under the empirical MDP, where the prior reward and transition functions $\Rewardprior,\Transitionprior$ are \emph{implicitly} defined by the initialization of $Q^{(0)}$.} Once this fixed point is determined, the value of $\pi$ may be approximated as $\avgstephat(\pi|\Dnset)= (1-\gamma)\cdot\E_{\visitDn(s_0),a_0\sim\pi(s_0)}[\Qhat^\pi(s_0,a_0)]$.
When the iterative procedure in~\eqref{eq:qe-backup} is performed via a regression over parameterized $Q$, this procedure is known as \emph{fitted $Q$-evaluation} (FQE). The reader may look to~\cite{voloshin2019empirical} for a review of a variety of instantiations of the direct method.

Although DM via either MB or QE is straightforward, it generally yields \emph{biased} estimates of $\avgstep(\pi)$: 
\begin{equation}
    \avgstep(\pi) \ne \E_{\Dnset}[\avgstephat(\pi|\Dnset)].
\end{equation}
Still, unbiased estimates are not completely necessary in practical risk-sensitive applications, where one would rather have access to accurate confidence intervals, and the bias of a single point estimate is irrelevant.
In the statistics literature, Efron's bootstrap (Algorithm~\ref{alg:bootstrap}) is widely used to provide asymptotically accurate confidence intervals, even when point estimates of the statistic are biased, and doing the same for DM methods has been proposed in the past~\cite{hanna2017bootstrapping}.
However, Efron's bootstrap is not always guaranteed to yield accurate confidence intervals~\cite{putter2012resampling,abadie2008failure}. 
In this paper, we will investigate conditions under which Efron's bootstrap applied to DM is guaranteed to yield accurate confidence intervals, and suggest mechanisms to improve the validity of the confidence intervals when these conditions do not hold.

Before getting into our main contributions, we list a few useful assumptions. For ease of exposition, we state these assumptions and our theoretical results with respect to \emph{countable} sets $\Sset$ and $\Aset$; this allows us to avoid technical details from measure theory.
\begin{assumption}[Bounded rewards]
\label{assumption:bounded_rewards}
   The rewards of the MDP are bounded by some finite constant $\Rmax$: $\nbr{r}_\infty\le \Rmax$.
 \end{assumption}
For the next assumption, we make use of the discounted on-policy distribution $\visitpi$, which measures the likelihood of the policy $\pi$ encountering state-action pair $(s,a)$ when interacting with $\mdp$~\cite{nachum2020reinforcement}:
\begin{equation}
    \label{eq:dpi}
    \visitpi(s,a) \defeq (1-\gamma)\cdot\sum_{t=0}^\infty \gamma^t\cdot \Pr[s_t=s,a_t=a~|~\init,\pi,\Reward,\Transition].
\end{equation}
 \begin{assumption}[Sufficient data coverage]
 \label{assumption:bounded_ratios}
   There exists $\epsilon>0$ such that for any $(s,a)$, $\visitpi(s,a)>0$ implies $\visitrb(s,a)>\epsilon$.
 \end{assumption}
As we will discuss later, Assumption~\ref{assumption:bounded_ratios} is very strong and often not satisfied in practice (e.g., in infinite state or action spaces). 

\begin{algorithm}
   \caption{Efron's non-parameteric, bias-corrected bootstrap~\cite{efron1987better}.}
\begin{algorithmic}
    \label{alg:bootstrap}
   \STATE {\bf Inputs}: A functional $F$, a desired confidence $1-\alpha$, a finite sample dataset $\Dset_n\defeq\{(s_0^{(j)}, s^{(j)}, a^{(j)}, r^{(j)}, s^{\prime(j)})\}_{j=1}^n$, number of bootstraps $b$ to use for percentile calculation.
   \vspace{2mm}
   \STATE \emph{\#\# Note: $F$ is a function from distributions over $(s_0,s,a,r,s')$ to $\R$. When applied to a finite dataset $\Dtilde$, it is understood to be applied to the empirical distribution $\visitDtilde$ determined by $\Dtilde$.}
   \vspace{2mm}
   \STATE Compute empirical estimate $\hat{y}\defeq F(\Dset_n)$. 
   \STATE Create $b$ bootstrapped datasets $\{\Dset_n^{(k)}\}_{k=1}^b$, each of $n$ elements sampled uniformly from $\Dset_n$.
   \STATE Compute bootstrapped estimates $\hat{y}_1\defeq F(\Dset_n^{(1)}),\dots,\hat{y}_b\defeq F(\Dset_n^{(b)})$.
   \STATE Compute $\alpha/2$ and $1-\alpha/2$ quantiles $z_{\alpha/2},z_{1-\alpha/2}$ of $\{\hat{y}_k - \hat{y}\}_{k=1}^b$.
   \vspace{2mm}
   \STATE {\bf Return} $C\defeq [\hat{y} - z_{1-\alpha/2}, \hat{y} - z_{\alpha/2}]$.
\end{algorithmic}
\end{algorithm}

\section{Investigating the Validity of Efron's Bootstrap}
We begin by presenting a theoretical result showing the validity of using Efron's bootstrap based on estimates $\avgstephat(\pi|\Dnset)$ prescribed by the direct method.
\begin{theorem}[Correctness of DM with bootstrapping]
\label{theorem:qe}
Under Assumptions~\ref{assumption:bounded_rewards},\ref{assumption:bounded_ratios}, the use of Algorithm~\ref{alg:bootstrap} with $F(\visitDn)\defeq \avgstephat(\pi|\Dnset)$ yields confidence intervals $C(\visitDn)$ which are asymptotically correct, in the sense that
\begin{equation}
    \Pr[\rho(\pi)\in C(\visitDn)] = 1 - \alpha - O_p(n^{-1/2}),
\end{equation}
where $O_p$ is used to denote \emph{order in probability}.
Additionally, the one-sided confidence intervals are asymptotically correct at rate $O_p(n^{-1/2})$. 
These asymptotic rates may be improved by using more sophisticated bootstrapping methods in place of Algorithm~\ref{alg:bootstrap}, such as BCa or ABC~\cite{diciccio1996bootstrap}.
\end{theorem}   
\begin{proof}
(Sketch) First, it is clear by the definition of $\visitrb$ in~\eqref{eq:def-d} and Assumption~\ref{assumption:bounded_ratios} that $F(\visitrb)=\avgstep(\pi)$. Thus it is left to show that bootstrap yields correct intervals around $F(\visitrb)$.
Sufficient conditions for correctness of Efron's bias-corrected bootstrap are known, and they are given by smoothness (specifically, Hadamard differentiability\footnote{See the appendix for a definition of Hadamard differentiability.}) of the functional $F$ evaluated in a neighborhood (i.e., a sufficiently small $L_\infty$ ball) around the true distribution $\visitrb$~\cite{wasserman2006all,politis2012subsampling,hall2013bootstrap}.
In the appendix, we show that under the assumption of bounded rewards (Assumption~\ref{assumption:bounded_rewards}) the derivative $F'(\visitDtilde)$ for general distribution $\visitDtilde$ satisfies 
\begin{equation}
\label{eq:f-deriv1}
||F'(\visitDtilde)||_\infty = O\left(||\visitpitilde/\visitDtilde||_\infty\right),
\end{equation}
where $\visitpitilde$ is the discounted on-policy distribution of $\pi$ under $\initDtilde,\RewardDtilde,\TransitionDtilde$.
When $\visitDtilde=\visitrb$, we have $||F'(\visitrb)||_\infty = O\left(||\visitpi/\visitrb||_\infty\right)$.
In the appendix, we show that $||\visitpitilde/\visitDtilde||_\infty$ is bounded within a sufficiently small neighborhood of $\visitrb$, given sufficient coverage of $\visitrb$ (Assumption~\ref{assumption:bounded_ratios}), and this completes the proof.
\end{proof}
\vspace{-3mm}

Although necessary conditions for the validity of Efron's bootstrap are not known in general, Hadamard differentiability is the key property typically used to prove validity. 
Our derivations make it clear that Assumption~\ref{assumption:bounded_ratios} is necesssary to ensure Hadamard differentiability of $F$; otherwise, a small change in $\visitrb$ may take $\visitpi$ out of the support of $\visitrb$, causing divergence in the derivative~\eqref{eq:f-deriv1}.
In contrast, a weaker variant of this assumption, $||\visitpi/\visitrb||_\infty=\Wmax<\infty$, which appears in previous OPE literature~\cite{nachum2019dualdice}, is not sufficiently strong to guarantee differentiability in the neighborhood of $\visitrb$. %
We encapsulate this in the following theorem.
\begin{theorem}[Necessity of Assumption~\ref{assumption:bounded_ratios}]
\label{theorem:bounded_ratios}
Suppose Assumption~\ref{assumption:bounded_rewards} holds and define functional $F(\visitDn)\defeq \avgstephat(\pi|\Dnset)$. There exists $\visitrb$ with uniformly bounded ratios $||\visitpi/\visitrb||_\infty=\Wmax<\infty$ such that $F$ is not Hadamard differentiable within any neighborhood of $\visitrb$.
\end{theorem}   
\begin{proof}
See the appendix.
\end{proof}
Theorem~\ref{theorem:bounded_ratios} is somewhat disappointing, as Assumption~\ref{assumption:bounded_ratios} is strong and often not satisfied in practice; in continuous state or action settings, it is almost never satisfied.

In addition to the need for Assumption~\ref{assumption:bounded_ratios}, the other major lacking of Theorem~\ref{theorem:qe} is that it only guarantees correct intervals \emph{asymptotically}. For any finite $n$, the confidence intervals yielded by Efron's bootstrap will generally exhibit \emph{under-coverage}, and in practice this can lead to overly confident confidence intervals. Indeed, in the extreme case of $n=1$, there will be no variation in the boostrapped estimates of $\avgstep$ leading to confidence intervals $C(\visitDn)$ that are single points.

In the following subsections, we elaborate on our suggested mechanisms for appropriately compensating for these two main theoretical shortcomings of Efron's bootstrap applied to DM.

\subsection{Regularizations for Insufficient Coverage}
To better understand the need for sufficient coverage, we can look at a simple scenario illustrated in Figure~\ref{fig:trajectories}.
If the data distribution includes $s_2$ but does not cover the action $\pi(s_2)$ chosen by the policy, then the estimates $\RewardDn(s_2,\pi(s_2)),\TransitionDn(s_2,\pi(s_2))$ will be set to the priors $\Rewardprior(s_2,\pi(s_2)),\Transitionprior(s_2,\pi(s_2))$.
However, if the data distribution includes state action pair $(s_2, \pi(s_2))$ with even a tiny probability, the estimates $\RewardDn(s_2,\pi(s_2)),\TransitionDn(s_2,\pi(s_2))$ are changed to their empirical estimates.
In general, this change is not smooth (i.e., not Hadamard differentiable) with respect to the underlying data distribution, and this leads to issues with the validity of Efron's bootstrap applied to $\avgstephat$.

\begin{wrapfigure}{R}{0.3\textwidth}
\begin{minipage}{0.3\textwidth}
\begin{center}
\vspace{-5mm}
\includegraphics[width=0.99\columnwidth]{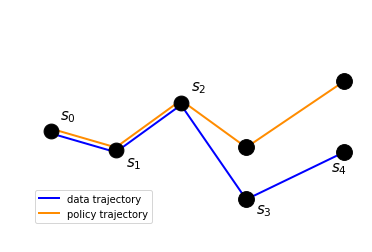}
\vspace{-5mm}
\end{center}
\caption{Trajectories of the policy may diverge from trajectories in the data.}
\vspace{-15mm}
\label{fig:trajectories}
\end{minipage}
\end{wrapfigure}

It is thus clear that to ensure validity of Efron's bootstrap, we require estimates $\RewardDn,\TransitionDn$ that are smoother around $\visitDn(s,a)\approx 0$.
For example, smoother empirical reward and transition functions may be found by defining biased reward and transitions in terms of some fixed $\kappa>0$,
\begin{align}
    \RewardBDn(r|s,a) &\defeq \frac{\visitDn(s, a, r) + \kappa\cdot\Rewardprior(r|s,a)}{\visitDn(s, a) + \kappa} \\
    \TransitionBDn(s'|s,a) &\defeq \frac{\visitDn(s, a, s') + \kappa\cdot\Transitionprior(s'|s,a)}{\visitDn(s, a) + \kappa}.
\end{align}
These biased functions would yield a \emph{regularized} DM estimate:
\begin{equation}
    \label{eq:biased-dm}
    \avgstepbiased(\pi|\Dnset) \defeq (1-\gamma)\cdot \E\left[\left.\sum_{t=0}^\infty \gamma^t\cdot r_t~\right|~ \initDn,\pi,\RewardBDn,\TransitionBDn\right].
\end{equation}
This estimator is provably amenable to statistical bootstrapping regardless of data coverage, although at the cost of providing intervals for a \emph{biased} estimate of $\avgstep(\pi)$, as stated by the following theorem.
\begin{theorem}[Correctness of regularized DM with bootstrapping]
\label{theorem:regularized}
Under Assumption~\ref{assumption:bounded_rewards}, the use of Algorithm~\ref{alg:bootstrap} with $F(\visitDn)\defeq \avgstepbiased(\pi|\Dnset)$ yields confidence intervals $C(\visitDn)$ which are asymptotically correct, in the sense that
\begin{equation}
    \Pr[\avgstepbiased(\pi|\visitrb)\in C(\visitDn)] = 1 - \alpha - O_p(n^{-1/2}).
\end{equation}
As for Theorem~\ref{theorem:qe}, the one-sided intervals converge at a rate $O_p(n^{-1/2})$ and these rates may be improved by using more sophisticated bootstrapping methods.
\end{theorem}   
\begin{proof}
See the appendix.
\end{proof}

For succinctness, we have expressed Theorem~\ref{theorem:regularized} in terms of the specific $\RewardBDn,\TransitionBDn$ defined above. In general, the guarantees of the theorem hold for any suitably smooth $\RewardBDn,\TransitionBDn$, i.e., reward and transition functions that are locally differentiable around $\visitrb$; see the appendix for details.
This more general result is promising for function approximation settings.
In such settings, when using model-based evaluation or fitted $Q$-evaluation, it is straightforward to smooth out the estimated reward and transition functions via a number of standard regularizations.
For example, in our experiments with neural network function approximators, we utilize standard weight decay, which acts as a regularization towards prior reward and transition functions implicitly defined by the network structure.

\subsection{Noisy Rewards}
Even with sufficient coverage or appropriate regularization, the computed confidence intervals will generally be over-confident and under-cover the true value, especially in low-data regimes. This is due to the fact that for finite $n$, the empirical variance of the functional $F$ over the bootstrapped datasets is in general an underestimate of the true variance. %

To incorporate additional variance, we propose to augment the dataset $\Dnset$ via perturbations applied to observed rewards,
\begin{multline}
    ~~~~~~~~~~~~~~~~~~~~~~~~~\widetilde{\Dset}_n \leftarrow \Dnset~~\cup~~ \{(s_0, s, a, r + \rnoise, s')~|~(s_0, s, a, r, s')\in\Dnset\} \\ \cup~~ \{(s_0, s, a, r - \rnoise, s')~|~(s_0, s, a, r, s')\in\Dnset\}.
\end{multline}
Note that the variance of the empirical dataset is increased to $\mathrm{Var}_{\widetilde{\Dset}_n}[r] = \frac{2}{3}\rnoise^2 + \mathrm{Var}_{\Dnset}[r]$.
Given the augmented dataset $\widetilde{\Dset}_n$, one may perform Algorithm~\ref{alg:bootstrap} as-is, sampling $b$ bootstrapped datasets each of $n$ elements.
This same technique of augmenting a dataset with noisy rewards has been used in the bandit literature as a way to perform better exploration~\cite{kveton2019perturbed,kveton2018garbage}. As in this previous literature, a large enough $\rnoise \ge \sqrt{\frac{3}{2}}\cdot(1-\gamma)^{-1}\cdot\Rmax$ would be sufficient to compensate for the inherent under-coverage in bootstrapping, although in practice a much smaller $\rnoise$ can still yield good coverage. 

With noisy rewards, we are able to compensate for the under-coverage of Theorems~\ref{theorem:qe} and~\ref{theorem:regularized}. However, this generally comes at the cost of over-coverage. In practice, the parameter $\rnoise$ provides a way to trade-off between safety in small data regimes and looseness of the confidence intervals.
In our experiments, we found that setting $\rnoise=0.25\cdot\sqrt{\mathrm{Var}_{\Dnset}[r]}$ provides a reasonable trade-off for our considered environments.
\section{Related Work}
Our paper focuses on producing confidence bounds for off-policy evaluation and therefore follows a long line of work on \emph{high-confidence policy evaluation} (HCOPE)~\cite{Thomas15HCPE}. 
Many of the existing methods for HCOPE focus on importance sampling (IS) based estimators, in which the rewards of a trajectory are re-weighted according to an inverse propensity ratio to yield an unbiased estimate of $\avgstep(\pi)$~\cite{precup2000eligibility}.
Given a dataset with several trajectories, one may derive several unbiased estimates and then use concentration inequalities to derive high-confidence lower and upper bounds on the true average~\cite{Thomas15HCPE}. Since these concentration inequalities typically require unbiased estimates, they are not applicable to the direct method. %

In terms of statistical bootstrapping, there have been several instances of its use for off-policy evaluation. Specifically,~\cite{Thomas15HCPI} combined statistical bootstrapping with IS to derive OPE confidence intervals. Unlike for DM, the validity of Efron's bootstrap with IS is straightforward, since the functional $F$ in this case is the standard mean.
We are aware of one previous instance in which statistical bootstrapping was used for high-confidence policy evaluation with DM; specifically,~\cite{hanna2017bootstrapping} proposes to use Efron's bootstrap in conjunction with model-based learning, similar to the present work. However, the validity of using Efron's bootstrap is not addressed in this previous work. 
The theoretical investigation we presented is a key contribution of our paper.
Notably, we found that the use of Efron's bootstrap directly is misguided without the use of strong assumptions, or alternatively, as we suggest, the use of mechanisms like regularization and noisy rewards.
Furthermore, our experimental work presents strong results on continuous control benchmarks, while previous work mostly focuses on tabular domains.

Outside of the narrow scope of HCOPE, the ideas behind Efron's bootstrap have inspired a number of existing RL algorithms. 
Specifically, statistical bootstrapping has been proposed as a mechanism for exploration; e.g., bootstrapped DQN~\cite{osband2016deep,osband2017deep}.
However, in practice, the type of bootstrapping performed in these algorithms is far from that prescribed by Efron's bootstrap. 
Usually, an ensemble of models is learned over the whole dataset, without any re-sampling or bias correction, and thus the theory behind bootstrap does not readily apply.
Although this simple paradigm has achieved impressive results on hard exploration environments~\cite{nachum2019does}, in our initial experiments for off-policy evaluation we found the naive ensembling approach to yield poor confidence intervals.
In the bandits literature, ideas from statistical bootstrapping have also been investigated as an exploration mechanism~\cite{kveton2018garbage,hao2019bootstrapping}. 
While we have focused on policy evaluation, extending the insights and derivations of the present paper to propose better algorithms for exploratory policy learning (or, conversely, safe policy learning) is an interesting avenue for future work.
\section{Experiments}

We evaluate our methods first in a discrete tabular domain, where we investigate how well the coverage of the estimated bootstrap intervals matches the intended coverage and show how reward noise can assist in low-data regimes. 
Sufficient coverage is not much of an issue in finite domains,\footnote{In finite domains, Assumption~\ref{assumption:bounded_ratios} reduces to $\visitpi(s,a)>0\Rightarrow \visitrb(s,a)>0$.} and so we continue to a more difficult set of continuous control tasks from OpenAI Gym~\cite{brockman2016openai}, where we evaluate the use of appropriately regularized function approximators in conjunction with bootstrapping and noisy rewards.

\subsection{Tabular Tasks}
We use Frozen Lake as a discrete domain for tabular experiments. 
In this environment, the agent navigates in a discrete world from a start state to a goal state. The environment dynamics are stochastic and some actions lead to episode terminations. We use $\gamma=0.999$. We use a target policy that is near-optimal in this domain. We collect an experience dataset using a behavior policy derived as the target policy injected with 0.2 $\epsilon$-greedy noise (this reduces the value of the policy $\avgstep(\pi)$ from about $0.0007$ to about $0.0002$). 
For this task, policy evaluation with DM with either MB or QE can be equivalently solved using the exact tabular method, so we plot a single variant labelled DM.

We present empirical results in Figure~\ref{fig:lake}. 
We plot the results of using Efron's bootstrap with DM to construct confidence intervals with confidence $1-\alpha$ across a number of dataset sizes.
The results here show empirical coverage of the estimated confidence intervals, as measured over 200 randomly sampled datasets (each dataset is then resampled repeatedly for computing bootstrap estimates).
We find that DM with bootstrapping is able to achieve near-correct empirical coverage as the dataset size grows. As suggested by the theory, the bootstrap typically underestimates the desired coverage, and this is severe in low-data regimes (when the number of episodes is less than $50$).
 
\begin{figure}[h]
 \begin{center}
 \setlength{\tabcolsep}{0pt}
 \renewcommand{\arraystretch}{0.7}
 \begin{tabular}{ccc}
 \includegraphics[width=0.3\columnwidth]{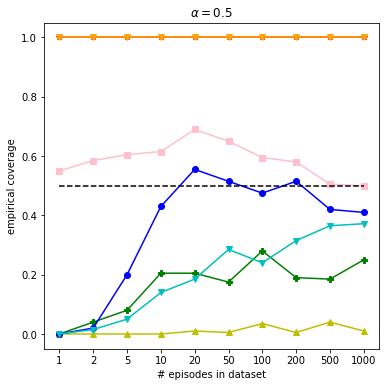} &
 \includegraphics[width=0.3\columnwidth]{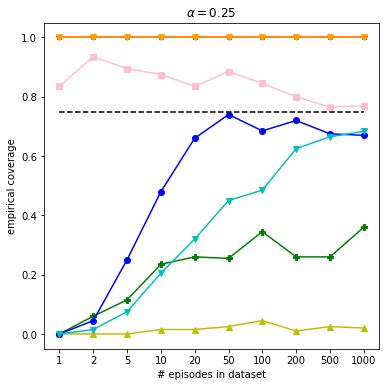} &
 \includegraphics[width=0.3\columnwidth]{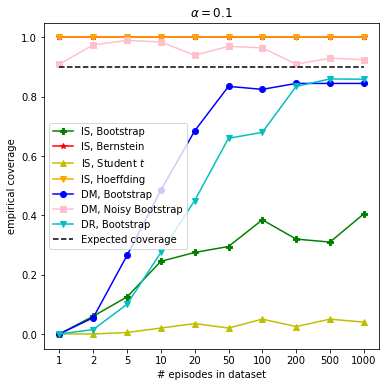} 
 \end{tabular}
\end{center}
\vspace{-5mm}
 \caption{Results on Frozen Lake across different confidences $1-\alpha$. Each plot shows the proportion of times the estimated confidence interval covers the true value of the policy, as measured over 200 separate trials.}
 \label{fig:lake}
 \end{figure}
 
We show the results of using noisy rewards to combat this low-data issue.
We perturb the rewards with $\rnoise=0.25\cdot\sqrt{\mathrm{Var}_{\Dnset}[r]}$.
The resulting difference in performance in the low-data regime is striking;
DM with noisy bootstrap is able to yield near-optimal coverage, although, as expected, it typically slightly overestimates the desired coverage.

As a point of comparison, we plot a number of other high-confidence policy evaluation methods: IS with bootstrapping, IS with empirical Bernstein's, IS with Student's $t$, IS with Hoeffding's, and doubly robust (DR) IS with bootstrapping (see~\cite{Thomas15HCPE,Thomas15HCPI,hanna2017bootstrapping}). We find that all of these previous methods mostly either severely underestimate or severely overestimate the desired coverage.
There is a potential for our proposed noisy rewards to be beneficial for some of these baselines as well (e.g., DR bootstrap), and this is a promising avenue for future work.

 \begin{figure}[h]
 \begin{center}
  \includegraphics[scale=0.35]{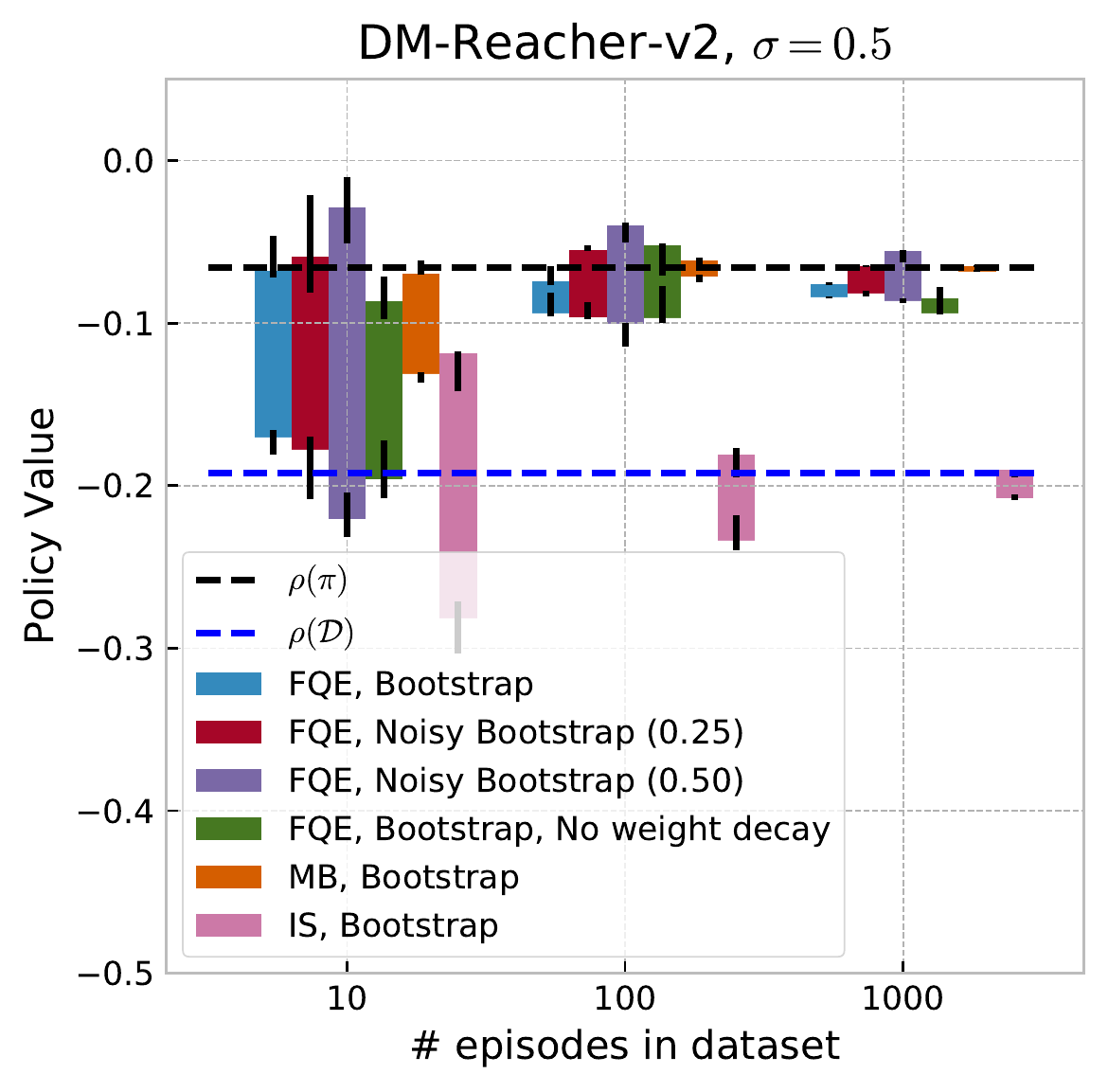}
 \includegraphics[scale=0.35]{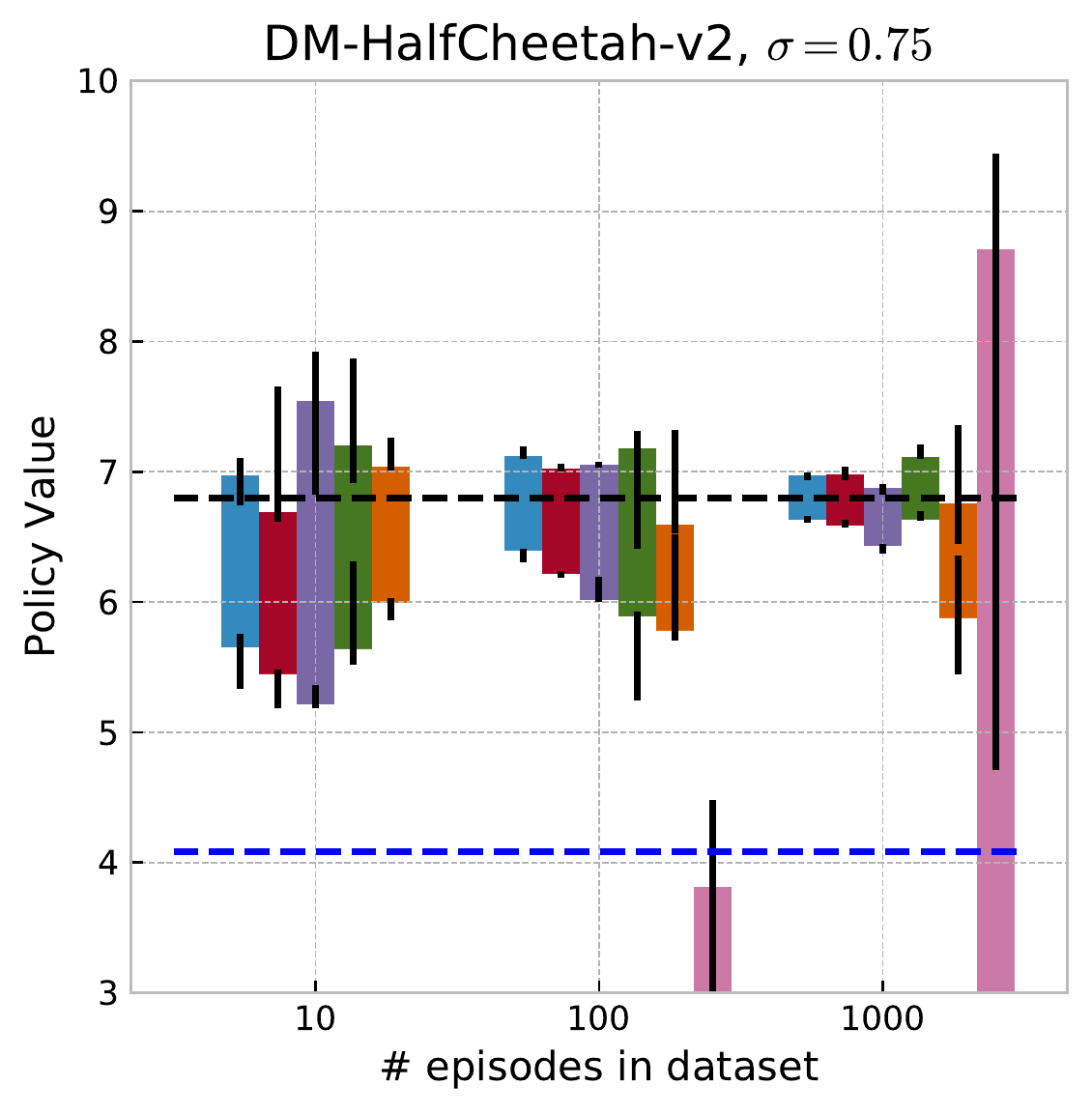}
\includegraphics[scale
=0.35]{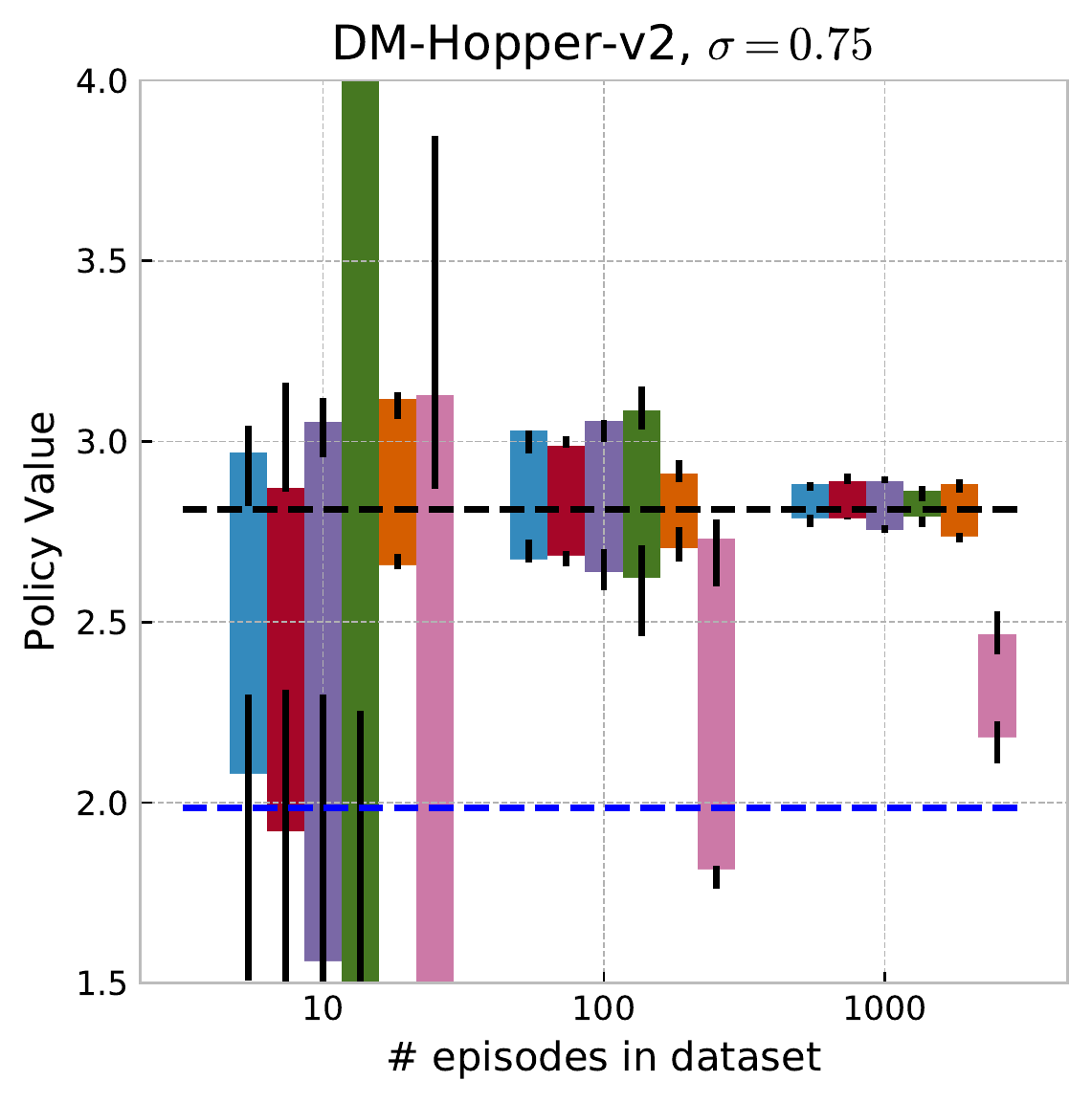}
\end{center}
\vspace{-5mm}
 \caption{Policy evaluation on continuous domains. For all methods, we plot estimated 95\% confidence intervals. For lower and upper bounds we plot a median and $25$th and $75$th percentiles (black vertical lines) over 5 seeds. We also plot values for the target policy value $\rho(\pi)$  and the behavior policy value $\rho(\mathcal{D})$. For FQE with noisy boostrapping, the noise scale corresponds to a coefficient applied to the standard deviation of observed rewards in the dataset. 
 Some of the variants (FQE without weight decay, IS bootstrap) at times produce intervals which are wholly outside the plotted range.
 }
 \label{fig:mujoco}
 \end{figure}

 \subsection{Continuous Control Tasks}
We now evaluate the use of bootstrapping on continuous control tasks from OpenAI gym~\cite{brockman2016openai}. Due to high computational demands, we focus on Reacher, HalfCheetah, and Hopper. We follow a protocol similar to \cite{nachum2019dualdice}. First, we generate a near-optimal policy by training SAC~\cite{haarnoja2018soft}. The target policy $\pi$ is set to be this near-optimal policy with fixed variance $\sigma^2=0.01$.
The datasets $\Dnset$ are sampled via a sub-optimal policy derived from the near-optimal policy with variance replaced with a fixed quantity $\sigma^2$ ($\sigma=0.5$ or $\sigma=0.75$ depending on the task). %
We train all networks for one million steps using stochastic gradient descent via the Adam optimizer~\cite{kingma2014adam} with learning rate $3\cdot10^{-4}$ and a minibatch size of $256$. 
As a form of regularization to combat issues with sufficient coverage, we apply weight decay (L2 regularization) equal to $10^{-5}$ to all methods, unless otherwise specified.

We present the computed intervals of FQE and MB in Figure~\ref{fig:mujoco}.
Focusing first on the effect of reward noise, we look at the ablation presented by the three FQE variants in these plots (see the appendix for an ablation over MB variants).
In extreme low data regimes (10 trajectories), the variance of vanilla FQE intervals is large and coverage of the true value suffers (especially for Reacher). 
With increased reward noise scales, coverage of the true value improves, but at the cost of a wider interval at times.

Next, we consider the issue of sufficient coverage.
By default, we apply L2 regularization to FQE. In Figure~\ref{fig:mujoco} we present a variant without L2 regularization.
We find the absence of this regularization to have a detrimental effect on performance. At times, the intervals computed by unregularized FQE are so inaccurate that they are outside the range of the plot.
We find that the regularized version of FQE exhibits more stable performance.
We found regularization of MB to also be crucial. The MB method plotted here uses L2 regularization on the weights and clips states and rewards generated during model-based rollouts. 
Although not plotted, we found that without these regularizations, the MB bootstrap intervals diverge. 
In some instances, we can see the consequences of these strong regularizations in terms of biased intervals that do not cover the true value, such as in HalfCheetah.

Overall, we conclude that DM approaches when using bootstrapping and our proposed mechanisms can yield strong performance in these difficult domains. 
Between FQE and MB, FQE appears to be better suited for these domains, although both methods show substantial improvement over existing approaches (IS with bootstrapping).\footnote{DR with bootstrap produces even worse intervals, and so we do not plot it.}

\section{Conclusion}
We have investigated the validity of Efron's bootstrap for computing confidence intervals with respect to the direct method (DM) for off-policy evaluation.
Our theoretical results show that Efron's bootstrap is valid given that specific conditions -- sufficient data size and sufficient coverage -- are satisfied.
While these conditions are often not satisfied in practice, there are a number of heuristic mechanisms that can be employed to mitigate their effects, although at a cost of overly conservative or biased intervals.
Still, empirically we find that these mechanisms can be used to yield impressive performance for OPE in challenging environments.
In the future, we hope to use the ideas and techniques presented here and apply them to policy optimization problems, where safety is also a key concern.

\section*{Broader Impact}
Our work focuses on the practically relevant problem of off-policy evaluation. Interestingly, our work reveals the potential issues with applying a well-known technique -- Efron's bootstrap -- without considering its validity.
Our work shows that Efron's bootstrap may often not be valid.
Although we propose mechanisms to remedy this, our solutions are not fool-proof. In a practical setting, where many of our assumptions may not hold, one must take special care when applying our method to mitigate risks of failure.

\begin{ack}
Thanks to Jonathan Tompson, Andy Zeng, Branislav Kveton, and others at Google Research for contributing helpful thoughts and discussions.
\end{ack}

\bibliography{references}

\begin{thebibliography}{10}

\bibitem{abadie2008failure}
Alberto Abadie and Guido~W Imbens.
\newblock On the failure of the bootstrap for matching estimators.
\newblock {\em Econometrica}, 76(6):1537--1557, 2008.

\bibitem{bradtke1996linear}
Steven~J Bradtke and Andrew~G Barto.
\newblock Linear least-squares algorithms for temporal difference learning.
\newblock {\em Machine learning}, 22(1-3):33--57, 1996.

\bibitem{brockman2016openai}
Greg Brockman, Vicki Cheung, Ludwig Pettersson, Jonas Schneider, John Schulman,
  Jie Tang, and Wojciech Zaremba.
\newblock Openai gym.
\newblock {\em arXiv preprint arXiv:1606.01540}, 2016.

\bibitem{diciccio1996bootstrap}
Thomas~J DiCiccio and Bradley Efron.
\newblock Bootstrap confidence intervals.
\newblock {\em Statistical science}, pages 189--212, 1996.

\bibitem{duan2020minimax}
Yaqi Duan and Mengdi Wang.
\newblock Minimax-optimal off-policy evaluation with linear function
  approximation.
\newblock {\em arXiv preprint arXiv:2002.09516}, 2020.

\bibitem{dudik2011doubly}
Miroslav Dud{\'\i}k, John Langford, and Lihong Li.
\newblock Doubly robust policy evaluation and learning.
\newblock {\em arXiv preprint arXiv:1103.4601}, 2011.

\bibitem{efron1987better}
Bradley Efron.
\newblock Better bootstrap confidence intervals.
\newblock {\em Journal of the American statistical Association},
  82(397):171--185, 1987.

\bibitem{haarnoja2018soft}
Tuomas Haarnoja, Aurick Zhou, Pieter Abbeel, and Sergey Levine.
\newblock Soft actor-critic: Off-policy maximum entropy deep reinforcement
  learning with a stochastic actor.
\newblock {\em arXiv preprint arXiv:1801.01290}, 2018.

\bibitem{hall2013bootstrap}
Peter Hall.
\newblock {\em The bootstrap and Edgeworth expansion}.
\newblock Springer Science \& Business Media, 2013.

\bibitem{hanna2016high}
J.~Hanna, P.~Stone, and S.~Niekum.
\newblock High confidence off-policy evaluation with models.
\newblock {\em arXiv preprint arXiv:1606.06126}, 2016.

\bibitem{hanna2017bootstrapping}
Josiah~P Hanna, Peter Stone, and Scott Niekum.
\newblock Bootstrapping with models: Confidence intervals for off-policy
  evaluation.
\newblock In {\em Thirty-First AAAI Conference on Artificial Intelligence},
  2017.

\bibitem{hao2019bootstrapping}
Botao Hao, Yasin~Abbasi Yadkori, Zheng Wen, and Guang Cheng.
\newblock Bootstrapping upper confidence bound.
\newblock In {\em Advances in Neural Information Processing Systems}, pages
  12123--12133, 2019.

\bibitem{jiang2015doubly}
Nan Jiang and Lihong Li.
\newblock Doubly robust off-policy value evaluation for reinforcement learning.
\newblock {\em arXiv preprint arXiv:1511.03722}, 2015.

\bibitem{kingma2014adam}
Diederik~P Kingma and Jimmy Ba.
\newblock Adam: A method for stochastic optimization.
\newblock {\em arXiv preprint arXiv:1412.6980}, 2014.

\bibitem{kveton2019perturbed}
Branislav Kveton, Csaba Szepesvari, Mohammad Ghavamzadeh, and Craig Boutilier.
\newblock Perturbed-history exploration in stochastic multi-armed bandits.
\newblock {\em arXiv preprint arXiv:1902.10089}, 2019.

\bibitem{kveton2018garbage}
Branislav Kveton, Csaba Szepesvari, Sharan Vaswani, Zheng Wen, Mohammad
  Ghavamzadeh, and Tor Lattimore.
\newblock Garbage in, reward out: Bootstrapping exploration in multi-armed
  bandits.
\newblock {\em arXiv preprint arXiv:1811.05154}, 2018.

\bibitem{li2011unbiased}
Lihong Li, Wei Chu, John Langford, and Xuanhui Wang.
\newblock Unbiased offline evaluation of contextual-bandit-based news article
  recommendation algorithms.
\newblock In {\em Proceedings of the fourth ACM international conference on Web
  search and data mining}, pages 297--306. ACM, 2011.

\bibitem{liao2019off}
Peng Liao, Predrag Klasnja, and Susan Murphy.
\newblock Off-policy estimation of long-term average outcomes with applications
  to mobile health.
\newblock {\em arXiv preprint arXiv:1912.13088}, 2019.

\bibitem{lillicrap2015continuous}
Timothy~P Lillicrap, Jonathan~J Hunt, Alexander Pritzel, Nicolas Heess, Tom
  Erez, Yuval Tassa, David Silver, and Daan Wierstra.
\newblock Continuous control with deep reinforcement learning.
\newblock {\em arXiv preprint arXiv:1509.02971}, 2015.

\bibitem{liu2018breaking}
Qiang Liu, Lihong Li, Ziyang Tang, and Dengyong Zhou.
\newblock Breaking the curse of horizon: Infinite-horizon off-policy
  estimation.
\newblock In {\em Advances in Neural Information Processing Systems}, pages
  5356--5366, 2018.

\bibitem{mandel2014offline}
Travis Mandel, Yun-En Liu, Sergey Levine, Emma Brunskill, and Zoran Popovic.
\newblock Offline policy evaluation across representations with applications to
  educational games.
\newblock In {\em Proceedings of the 2014 international conference on
  Autonomous agents and multi-agent systems}, pages 1077--1084. International
  Foundation for Autonomous Agents and Multiagent Systems, 2014.

\bibitem{murphy2001marginal}
Susan~A Murphy, Mark~J van~der Laan, James~M Robins, and Conduct Problems
  Prevention~Research Group.
\newblock Marginal mean models for dynamic regimes.
\newblock {\em Journal of the American Statistical Association},
  96(456):1410--1423, 2001.

\bibitem{nachum2019dualdice}
Ofir Nachum, Yinlam Chow, Bo~Dai, and Lihong Li.
\newblock Dualdice: Behavior-agnostic estimation of discounted stationary
  distribution corrections.
\newblock {\em arXiv preprint arXiv:1906.04733}, 2019.

\bibitem{nachum2020reinforcement}
Ofir Nachum and Bo~Dai.
\newblock Reinforcement learning via fenchel-rockafellar duality.
\newblock {\em arXiv preprint arXiv:2001.01866}, 2020.

\bibitem{nachum2019does}
Ofir Nachum, Haoran Tang, Xingyu Lu, Shixiang Gu, Honglak Lee, and Sergey
  Levine.
\newblock Why does hierarchy (sometimes) work so well in reinforcement
  learning?
\newblock {\em arXiv preprint arXiv:1909.10618}, 2019.

\bibitem{osband2016deep}
Ian Osband, Charles Blundell, Alexander Pritzel, and Benjamin Van~Roy.
\newblock Deep exploration via bootstrapped dqn.
\newblock In {\em Advances in neural information processing systems}, pages
  4026--4034, 2016.

\bibitem{osband2017deep}
Ian Osband, Daniel Russo, Zheng Wen, and Benjamin Van~Roy.
\newblock Deep exploration via randomized value functions.
\newblock {\em Journal of Machine Learning Research}, 2017.

\bibitem{paine2020hyperparameter}
Tom~Le Paine, Cosmin Paduraru, Andrea Michi, Caglar Gulcehre, Konrad Zolna,
  Alexander Novikov, Ziyu Wang, and Nando de~Freitas.
\newblock Hyperparameter selection for offline reinforcement learning, 2020.

\bibitem{politis2012subsampling}
Dimitris~N Politis, Joseph~P Romano, and Michael Wolf.
\newblock {\em Subsampling}.
\newblock Springer Science \& Business Media, 2012.

\bibitem{precup2000eligibility}
Doina Precup.
\newblock Eligibility traces for off-policy policy evaluation.
\newblock {\em Computer Science Department Faculty Publication Series},
  page~80, 2000.

\bibitem{puterman1994markov}
Martin~L Puterman.
\newblock Markov decision processes: Discrete stochastic dynamic programming.
\newblock 1994.

\bibitem{putter2012resampling}
Hein Putter and Willem~R Van~Zwet.
\newblock Resampling: consistency of substitution estimators.
\newblock In {\em Selected Works of Willem van Zwet}, pages 245--266. Springer,
  2012.

\bibitem{sutton2000policy}
Richard~S Sutton, David~A McAllester, Satinder~P Singh, and Yishay Mansour.
\newblock Policy gradient methods for reinforcement learning with function
  approximation.
\newblock In {\em Advances in neural information processing systems}, pages
  1057--1063, 2000.

\bibitem{Swaminathan17OP}
A.~Swaminathan, A.~Krishnamurthy, A.~Agarwal, M.~Dud{\'i}k, J.~Langford,
  D.~Jose, and I.~Zitouni.
\newblock Off-policy evaluation for slate recommendation.
\newblock In {\em Proceedings of the 31st International Conference on Neural
  Information Processing Systems}, pages 3635--3645, 2017.

\bibitem{Thomas16DE}
P.~Thomas and E.~Brunskill.
\newblock Data-efficient off-policy policy evaluation for reinforcement
  learning.
\newblock In {\em Proceedings of the 33rd International Conference on Machine
  Learning}, pages 2139--2148, 2016.

\bibitem{Thomas15HCPE}
P.~Thomas, G.~Theocharous, and M.~Ghavamzadeh.
\newblock High confidence off-policy evaluation.
\newblock In {\em Proceedings of the 29th Conference on Artificial
  Intelligence}, 2015.

\bibitem{Thomas15HCPI}
P.~Thomas, G.~Theocharous, and M.~Ghavamzadeh.
\newblock High confidence policy improvement.
\newblock In {\em Proceedings of the 32nd International Conference on Machine
  Learning}, pages 2380--2388, 2015.

\bibitem{thomas2015safe}
Philip~S Thomas.
\newblock {\em Safe reinforcement learning}.
\newblock PhD thesis, University of Massachusetts Libraries, 2015.

\bibitem{voloshin2019empirical}
Cameron Voloshin, Hoang~M Le, Nan Jiang, and Yisong Yue.
\newblock Empirical study of off-policy policy evaluation for reinforcement
  learning.
\newblock {\em arXiv preprint arXiv:1911.06854}, 2019.

\bibitem{wasserman2006all}
Larry Wasserman.
\newblock {\em All of nonparametric statistics}.
\newblock Springer Science \& Business Media, 2006.

\end{thebibliography}
\bibliographystyle{plain}

\newpage
\appendix
\section{Proofs}

\subsection{Hadamard Differentiability}
We provide a definition of Hadamard differentiability, which is a key property for showing validity of Efron's bootstrap.
The following is paraphrased from~\cite{wasserman2006all}.
\begin{definition}
Suppose $F$ is a functional mapping distributions over $\mathcal{P}\defeq \Sset\times \Sset\times\Aset\times\R\times\Sset$ (i.e., distributions of tuples $(s_0,s,a,r,s')$) to $\R$.
Denote $\mathcal{P}_L$ as the the linear space generated by $\mathcal{P}$.
The functional $F$ is said to be \textbf{Hadamard differentiable} at $\visitDtilde\in\mathcal{P}$ if there exists a linear functional $L_{\Dtilde}$ on $\mathcal{P}_L$
such that for any $\epsilon_n \to 0$ and
$P, P_1,P_2,P_3,\dots \in \mathcal{P}_L$ such that $\|P_n - P\|_\infty \to 0$ and $\visitDtilde + \epsilon_n P_n \in \mathcal{P}$,
\begin{equation}
    \lim_{n\to\infty} \left|
    \frac{F(\visitDtilde +\epsilon_n P_n) - F(\visitDtilde)}{\epsilon_n} - L_{\Dtilde}(P)
    \right| = 0.
\end{equation}
\end{definition}

\subsection{Proof of Theorem~\ref{theorem:qe}}
\label{sec:proof1}
As in the main text, we use $\initDtilde,\RewardDtilde,\TransitionDtilde$ to denote the initial state, conditional reward, and conditional transition distributions observed in $\visitDtilde$.
Furthermore, let $\AvgRewardDtilde = \E_{\RewardDtilde}[r]$.
For ease of notation, we will use matrix notation and assume finite state and action spaces (an extension to Hilbert spaces with linear operators is straightforward). The functional $F$ may be expressed as,
\begin{equation}
    \label{eq:def-f}
    F(\visitDtilde) \defeq (1-\gamma)\cdot \AvgRewardDtilde^{T} (I - \gamma \Pi \TransitionDtilde)^{-1} \Pi\initDtilde,
\end{equation}
where we use $\Pi$ to denote the matrix mapping distributions over states to distributions over state actions, where actions as sampled by $\pi$.

Note that, assuming $\visitpi(s,a)>0\Rightarrow \visitrb(s,a)>0$, the components of this expression for $F$ at $\visitDtilde=\visitrb$ yield the $Q^\pi$-values and on-policy distribution $\visitpi$. Specifically,
\begin{align}
    \label{eq:visitpi-matrix}
    \visitpi &= (1-\gamma)(I - \gamma \Pi \Transition)^{-1} \Pi\init, \\
    \label{eq:qpi-matrix}
    Q^\pi &= \overline{\Reward}^{T} (I - \gamma \Pi \Transition)^{-1}.
\end{align}
For general $\visitDtilde$, these expressions will yield $\visitpitilde$, the on-policy distribution in the empirical MDP, and $\Qpitilde$, the $Q^\pi$ values in the empirical MDP, respectively.

As mentioned in the proof sketch, the validity of Theorem~\ref{theorem:qe} rests on the Hadamard differentiability of $F(\visitDtilde)$ for all $\visitDtilde$ in a neighborhood around $\visitrb$.
In addition to local Hadamard differentiability, one must also have that the derivative linear functional $L_{\Dset}$ satisfy
\begin{equation}
    0 < \E_{(s_0,s,a,r,s')\sim\visitrb}\left[L_{\Dset}(\delta_{(s_0,s,a,r,s')}-\visitrb)^2\right] < \infty.
\end{equation}
See Theorems 3.19 and 3.21 in~\cite{wasserman2006all} for more information.
In the text below, we will show that $F$ is indeed Hadamard differentiable with derivative satisfying
\begin{equation}
    L_{\Dset}(\delta_{(s_0,s,a,r,s')}-\visitrb) = O(\visitpi(s,a) / \visitrb(s,a)).
    \label{eq:linear-func}
\end{equation}
The result~\eqref{eq:linear-func} in conjunction with Assumption~\ref{assumption:bounded_ratios} will immediately make it clear that $\E_{(s_0,s,a,r,s')\sim\visitrb}\left[L_{\Dset}(\delta_{(s_0,s,a,r,s')}-\visitrb)^2\right] < \infty$. Moreover, the linear nature of the functional $F$ with respect to $\RewardDtilde$ makes it clear that $0<\E_{(s_0,s,a,r,s')\sim\visitrb}\left[L_{\Dset}(\delta_{(s_0,s,a,r,s')}-\visitrb)^2\right]$, thus showing the validity of the bootstrap.

We now continue to characterize the linear functional $L_{\Dset}$. We will first derive, via standard Frechet differentiation, the derivatives of $F(\visitDtilde)$ with respect to $\AvgRewardDtilde$, $\initDtilde$, and $\TransitionDtilde$, for $\visitDtilde$ that satisfy Assumption~\ref{assumption:bounded_rewards}. We will later use these results in conjunction with Assumption~\ref{assumption:bounded_ratios} to show the Hadamard differentiability of $F$ with respect to $\visitDtilde$ in a ball around $\visitrb$.
\begin{itemize}
    \item $\AvgRewardDtilde$: It is clear from~\eqref{eq:visitpi-matrix} that $\partial F / \partial \AvgRewardDtilde = \visitpitilde$.
    \item $\initDtilde$: It is clear from~\eqref{eq:qpi-matrix} that $\partial F / \partial \initDtilde = (1-\gamma)\Qpitilde \Pi$.
    \item $\TransitionDtilde$: This derivation is not as trivial as the previous two. Still, it may be approached in a straightforward manner by utilizing the policy gradient theorem~\cite{sutton2000policy}. Although the policy gradient theorem is typically used to derive gradients of $F$ with respect to $\Pi$, we may apply it here, interpreting $\TransitionDtilde$ as the stationary ``policy'' whose gradient we wish to calculate (``transitions'' are now between state-action pairs and the ``actions'' are choices of next states). Specifically, we may re-write~\eqref{eq:def-f} as
    \begin{equation}
        K + (1-\gamma)\cdot(\AvgRewardDtilde^T\Pi) (I - \gamma \TransitionDtilde\Pi)^{-1} \TransitionDtilde (\Pi\initDtilde),
    \end{equation}
    where $K$ is constant with respect to $\TransitionDtilde$.
    This way, we deduce that $\frac{\partial F}{\partial \TransitionDtilde(s'|s,a)} = \visitpitilde(s,a) \cdot \E_{a'\sim\pi(s')}[\Qpitilde(s',a')]$ for all $s,a,s'$.
\end{itemize}

With these three partial derivatives calculated, we may continue to show differentiability of $F$ in a neighborhood around $\visitrb$.
Without loss of generality, we assume that $\visitpi$ has full support; if not, we may simply ignore all tuples outside of the support, since they do not affect $\avgstep(\pi)$ or $\avgstephat(\pi)$ (note that by Assumption~\ref{assumption:bounded_ratios} this means $\visitrb$ also has full support).

Now we continue to characterize the derivative of $F$. Denote the derivative of $F$ by $F'$, where $F'(\visitDtilde)$ is defined to be the linear functional satisfying,
\begin{equation}
    \label{eq:directional-deriv}
    \hspace{-2mm}
    \left\langle F'(\visitDtilde),
    \delta_{(s_0^*,s^*,a^*,r^*,s^{\prime*})} - \visitDtilde \right\rangle = \lim_{t\to0} \frac{1}{t}\left(F((1-t)\cdot\visitDtilde + t\cdot\delta_{(s_0^*,s^*,a^*,r^*,s^{\prime*})}) - F(\visitDtilde)\right),
    \hspace{-2mm}
\end{equation}
for all tuples $(s_0^*,s^*,a^*,r^*,s^{\prime*})$.
We analyze the behavior of these directional limits. 
We again split our analysis into three parts:
\begin{itemize}
    \item Influence of $r^*$. The influence of $r^*$ is in the empirical average reward function at $s^*,a^*$: $\AvgRewardDtilde(s^*,a^*)$. At a change of $t$, this value is updated to
    \begin{equation}
        \label{eq:rr-deriv}
        \frac{(1-t)\visitDtilde(s^*,a^*) \AvgRewardDtilde(s^*,a^*)+t r^*}{(1-t)\visitDtilde(s^*,a^*) + t}.
    \end{equation}
    The derivative of this expression at $t=0$ is $\frac{-\AvgRewardDtilde(s^*,a^*) + r^*}{\visitDtilde(s^*,a^*)}$. Combined with the partial derivative computed earlier, we find the total influence on $F$ is $\frac{\visitpitilde(s^*,a^*)}{\visitDtilde(s^*,a^*)} (-\AvgRewardDtilde(s^*,a^*) + r^*)\cdot t$ as $t\to0$. 
    
    \item Influence of $s_0^*$. The influence of $s_0^*$ is in the empirical initial state distribution $\initDtilde$, which is updated to $(1-t)\initDtilde + t\delta_{s_0^*}$. To deduce the influence on $F$, we combine with the partial derivative computed earlier, and find the change in $F$ to be $\left(-\avgstephat(\pi|\visitDtilde) + (1-\gamma)\E_{a_0\sim\pi(s_0^*)}[\Qpitilde(s_0^*,a_0)]\right)\cdot t$.
    
    \item Influence of $s^{\prime*}$. As for the reward, the influence here is in the empirical transition probabilities $\TransitionDtilde(s' | s^*,a^*)$, which is updated to
    \begin{equation}
        \label{eq:tt-deriv}
        \frac{(1-t)\visitDtilde(s^*,a^*)\TransitionDtilde(s'|s^*,a^*) + t\delta_{s^{\prime*}}(s')}{(1-t)\visitDtilde(s^*,a^*) + t}.
    \end{equation}
    The derivative of this expression at $t=0$ is $\frac{-\TransitionDtilde(s'|s^*,a^*) + \delta_{s^{\prime*}}(s')}{\visitDtilde(s^*,a^*)}$. Combining this with the known partials of $F$ with respect to $\TransitionDtilde$, we find that the total influence on $F$ is $\frac{\visitpitilde(s^*,a^*)}{\visitDtilde(s^*,a^*)} \left(-\E_{s'\sim\TransitionDtilde(s^*,a^*),a'\sim\pi(s')}[\Qpitilde(s',a')] + \E_{a'\sim\pi(s^{\prime*})}[\Qpitilde(s^{\prime*},a')]\right)\cdot t$ as $t\to0$
\end{itemize}
We may see that each of these influences on $F$ are linear in $t$. 
By Assumption~\ref{assumption:bounded_rewards}, $r^*$ is uniformly bounded, as are $\AvgRewardDtilde$ and $\Qpitilde$.
Thus, in conjunction with the Riesz representation theorem, we deduce that the derivative $F'$ satisfies
\begin{equation}
    \label{eq:f-deriv}
    ||F'(\visitDtilde)||_\infty = O\left( \left|\left|\frac{\visitpitilde}{\visitDtilde}\right|\right|_\infty\right).
\end{equation}
Now consider an arbitrary distribution $\visitEtilde$ and the directional limit
\begin{equation}
    \lim_{t\to0}\frac{1}{t}\left( F((1-t)\cdot\visitDtilde + t\cdot\visitEtilde) - F(\visitDtilde)\right).
\end{equation}
Analogous to the derivations above, we may find,
\begin{itemize}
    \item The empirical average reward $\AvgRewardDtilde(s, a)$ at a change of $t$ is updated to
    \begin{equation}
        \label{eq:rr-deriv2}
        \frac{(1-t)\visitDtilde(s,a) \AvgRewardDtilde(s,a)+t\visitEtilde(s,a)\AvgRewardEtilde(s,a)}{(1-t)\visitDtilde(s,a) + t\visitEtilde(s,a)}.
    \end{equation}
    \item The empirical initial state distribution at a change of $t$ is updated to
    \begin{equation}
        \label{eq:ii-deriv2}
        (1-t)\initDtilde + t\initEtilde.
    \end{equation}
    \item The empirical transition probabilities $\TransitionDtilde(s'|s,a)$ at a change of $t$ are updated to
    \begin{equation}
        \label{eq:tt-deriv2}
        \frac{(1-t)\visitDtilde(s,a)\TransitionDtilde(s'|s,a) + t\visitEtilde(s,a)\TransitionEtilde(s'|s,a)}{(1-t)\visitDtilde(s,a) + t\visitEtilde}.
    \end{equation}
\end{itemize}
By considering the limits of~\eqref{eq:rr-deriv2},~\eqref{eq:ii-deriv2},~\eqref{eq:tt-deriv2} as $t\to0$, it is clear that 
\begin{equation}
    \label{eq:directional-deriv2}
    \left\langle F'(\visitDtilde),
    \visitEtilde - \visitDtilde \right\rangle = \lim_{t\to0} \frac{1}{t}\left(F((1-t)\cdot\visitDtilde + t\cdot\visitEtilde) - F(\visitDtilde)\right).
\end{equation}
To show Hadamard differentiability, we invoke Assumption~\ref{assumption:bounded_ratios}, which implies that there exists a sufficiently small $\zeta=\epsilon/2$ such that the $L_\infty$ ball centered at $\visitrb$ with radius $\zeta$ has uniformly bounded $||\visitpi/\visitDtilde||_\infty$. 
Since the support of $\visitpitilde$ is contained within the support of $\visitpi$, this means that the same ball has uniformly bounded $||\visitpitilde/\visitDtilde||_\infty$.
Moreover, it is clear that within this ball $\visitDtilde>\epsilon/2$ uniformly, and so the directional derivatives of~\eqref{eq:rr-deriv2},~\eqref{eq:ii-deriv2}, and~\eqref{eq:tt-deriv2} converge uniformly with $t\cdot\|\visitEtilde\|_\infty$.
Thus, there exists a sufficiently small ball around $\visitrb$ within which $F$ is Hadamard differentiable. This completes our proof.

\subsection{Proof of Theorem~\ref{theorem:bounded_ratios}}
First, a brief sketch: If Assumption~\ref{assumption:bounded_ratios} does not hold, then for any $L_\infty$ ball, one may find a distribution near $\visitrb$ outside of the support of $\pi$, and this will cause discontinuities in $F$. 

Now more concretely: Consider an MDP with state space $\{s_{\mathrm{start}}, s_{\mathrm{term}},  s_1,s_2,\dots\}$. The MDP's initial state distribution is $\init \defeq \delta_{s_{\mathrm{start}}}$.
The MDP has a single action $a$ and the transition function is defined as,
\begin{align}
    \Transition(s_n|s_{\mathrm{start}},a) & = \frac{6}{\pi^2 n^2}, \\
    \Transition(s_{\mathrm{start}}|s_{\mathrm{start}},a) & = 0, \\
    \Transition(s_{\mathrm{term}}|s_{\mathrm{start}},a) & = 0, \\
    \Transition(s_n, a) &= \delta_{s_{\mathrm{term}}}, \\
    \Transition(s_{\mathrm{term}}, a) &= \delta_{s_{\mathrm{term}}}
\end{align}
The reward function is defined as
\begin{align}
    \Reward(s_{\mathrm{start}},a) &= \delta_0, \\
    \Reward(s_n,a) &= \delta_0, \\
    \Reward(s_{\mathrm{term}},a) &= \delta_1.
\end{align}
Define prior reward and transition functions
\begin{align}
    \Transitionprior(s, a) &\defeq \delta_{s_{\mathrm{term}}}, \\
    \Rewardprior(s, a) &\defeq \delta_1.
\end{align}
Let policy $\pi$ be a policy on this MDP (there exists only one).
Consider $\gamma=0.5$.
Thus we have
\begin{equation}
    \avgstep(\pi) = \frac{1}{4},
\end{equation}
\begin{equation}
    \visitpi(s_n, a) = \frac{3}{2\pi^2 n^2}.
\end{equation}
Let $\visitrb$ be defined as $\visitrb \defeq \visitpi$.
It is clear that $\visitrb$ satisfies $\|\visitpi / \visitrb\|_\infty = 1 < \infty$ but that Assumption~\ref{assumption:bounded_ratios} does not hold.

Now consider any $L_\infty$ ball around $\visitrb$.
Suppose this ball has radius $\zeta>0$ and let $N$ be such that $\frac{3}{2\pi^2 N^2} < \zeta$. We may define the distribution
\begin{equation}
    \visitDtilde \defeq \visitrb - \frac{3}{2\pi^2 N^2} \delta_{(s_{\mathrm{start}}, s_N, a, 0, s_{\mathrm{term}})} + \frac{3}{2\pi^2 N^2} \delta_{(s_{\mathrm{start}}, s_{\mathrm{term}}, a, 1, s_{\mathrm{term}})}.
\end{equation}
It is clear that $\visitDtilde$ is within the $L_\infty$ ball and $\visitDtilde(s_N, a) = 0$.
Thus, $\RewardDtilde(s_N,a)=\Rewardprior$ and so
\begin{equation}
    F(\visitDtilde) = \frac{1}{4} + \frac{3}{2\pi^2 N^2}.
\end{equation}
Now we define
\begin{equation}
    P \defeq \delta_{(s_{\mathrm{start}}, s_1, a, 0, s_{\mathrm{term}})} - \visitDtilde,
\end{equation}
\begin{equation}
    \epsilon_n \defeq \frac{1}{n}.
\end{equation}
It is clear that a change $\visitDtilde\to \visitDtilde + \epsilon_n\cdot P$ would not change the empirical reward or transition functions, and so we have,
\begin{equation}
    \lim_{n\to\infty} \frac{1}{\epsilon_n}(F(\visitDtilde + \epsilon_n\cdot P) - F(\visitDtilde)) = 0.
\end{equation}
We may also consider a sequence $\{P_n\}_{n=1}^\infty$ defined as
\begin{equation}
    P_n\defeq \frac{1}{n}\cdot\delta_{(s_{\mathrm{start}}, s_N, a, 0, s_{\mathrm{term}})} + \left(1-\frac{1}{n}\right)\delta_{(s_{\mathrm{start}}, s_1, a, 0, s_{\mathrm{term}})} - \visitDtilde.
\end{equation}
Clearly $\lim_{n\to\infty} P_n = P$. However, $P_n$ changes the empirical reward distribution at $(s_N, a)$, and this causes
\begin{equation}
    \lim_{n\to\infty} \frac{1}{\epsilon_n} (F(\visitDtilde + \epsilon_n\cdot P_n) - F(\visitDtilde)) = \lim_{n\to\infty} \frac{1}{\epsilon_n}\cdot\frac{3}{2\pi^2 N^2} = \infty.
\end{equation}
Thus, $F$ is not Hadamard differentiable at $\visitDtilde$.

\subsection{Proof of Theorem~\ref{theorem:regularized}}
We prove a more useful generalization of this theorem, stated below:
\paragraph{Generalized Theorem~\ref{theorem:regularized}}
Suppose $\RewardBDtilde,\TransitionBDtilde$ are reward and transition probability functions defined with respect to general distributions $\visitDtilde$ and that these functions are differentiable with respect to $\visitDtilde$ in a neighborhood around $\visitrb$ with uniformly bounded derivatives.
Under Assumption~\ref{assumption:bounded_rewards}, the use of Algorithm~\ref{alg:bootstrap} with $F(\visitDn)\defeq \avgstepbiased(\pi|\Dnset)$ yields confidence intervals $C(\visitDn)$ which are asymptotically correct, in the sense that
\begin{equation}
    \Pr[\avgstepbiased(\pi|\visitrb)\in C(\visitDn)] = 1 - \alpha - O_p(n^{-1/2}).
\end{equation}
As for Theorem~\ref{theorem:qe}, the one-sided intervals converge at a rate $O_p(n^{-1/2})$ and these rates may be improved by using more sophisticated bootstrapping methods.

\paragraph{Proof}
The proof is straightforward given the derivations in Section~\ref{sec:proof1}.
Specifically, analogous to Section~\ref{sec:proof1} one may readily show that
\begin{align}
    \partial F / \partial \overline{\Reward}_{\Dtilde}^\kappa &= \visitpitilde \\
    \frac{\partial F}{\partial \TransitionDtilde^\kappa(s'|s,a)} &= \visitpitilde(s,a) \cdot \E_{a'\sim\pi(s')}[\Qpitilde(s',a')].
\end{align}
Using chain rule with the assumption of differentiability of $\RewardBDtilde,\TransitionBDtilde$ then immediately shows that $F'$ is well-defined and thus $F$ is appropriately differentiable around $\visitrb$.

\newpage
\subsection{Additional Experiments}

In this section we provide additional experiments for Model Based Policy Evaluation. In particular, we demonstrate
that the scale of the noise has a similar effect for MB policy evaluation as for Fitted Q-Evaluation (see Figure \ref{app:mujoco}).

 \begin{figure}[h]
 \begin{center}
  \includegraphics[scale=0.375]{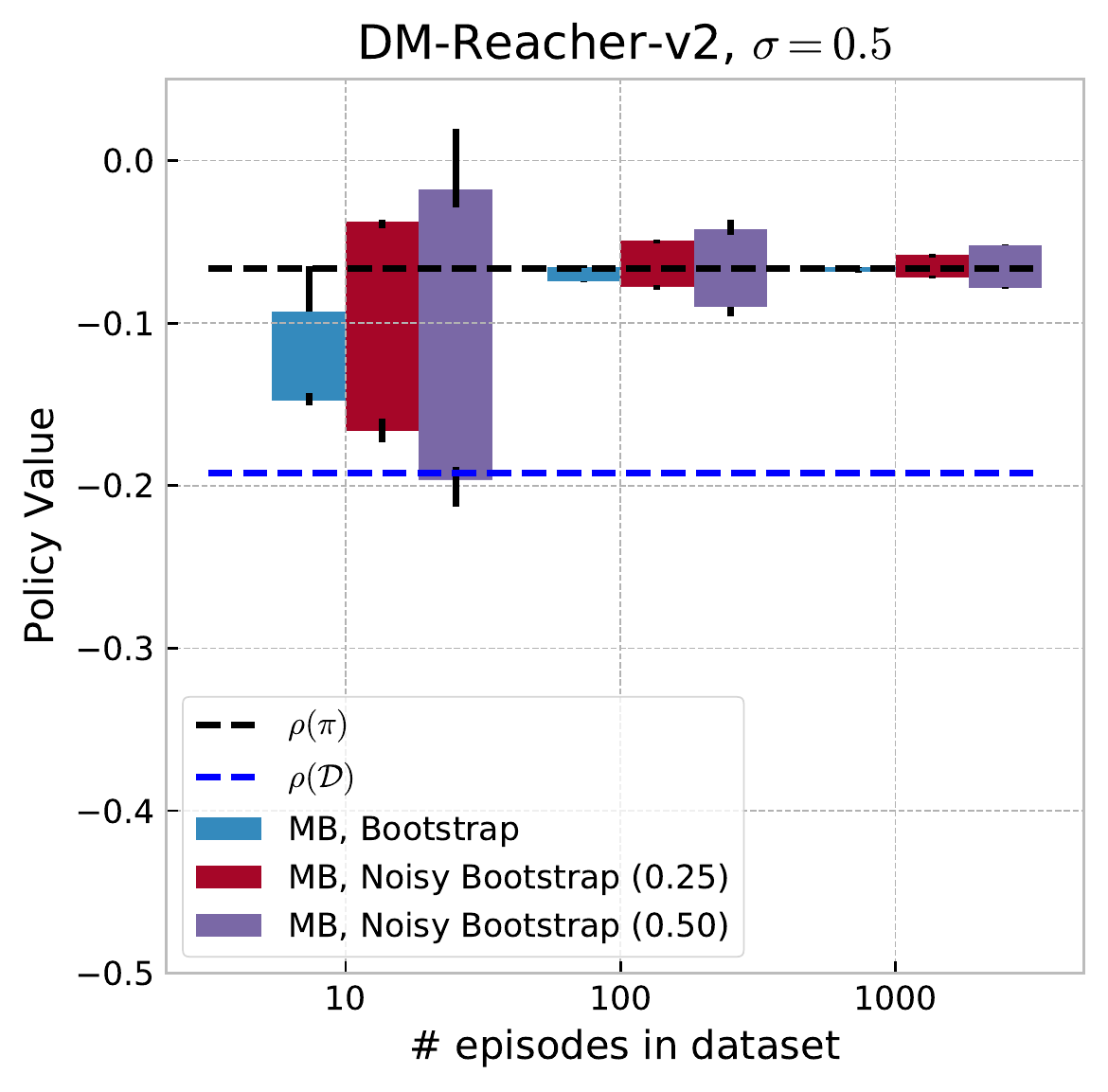}
 \includegraphics[scale=0.375]{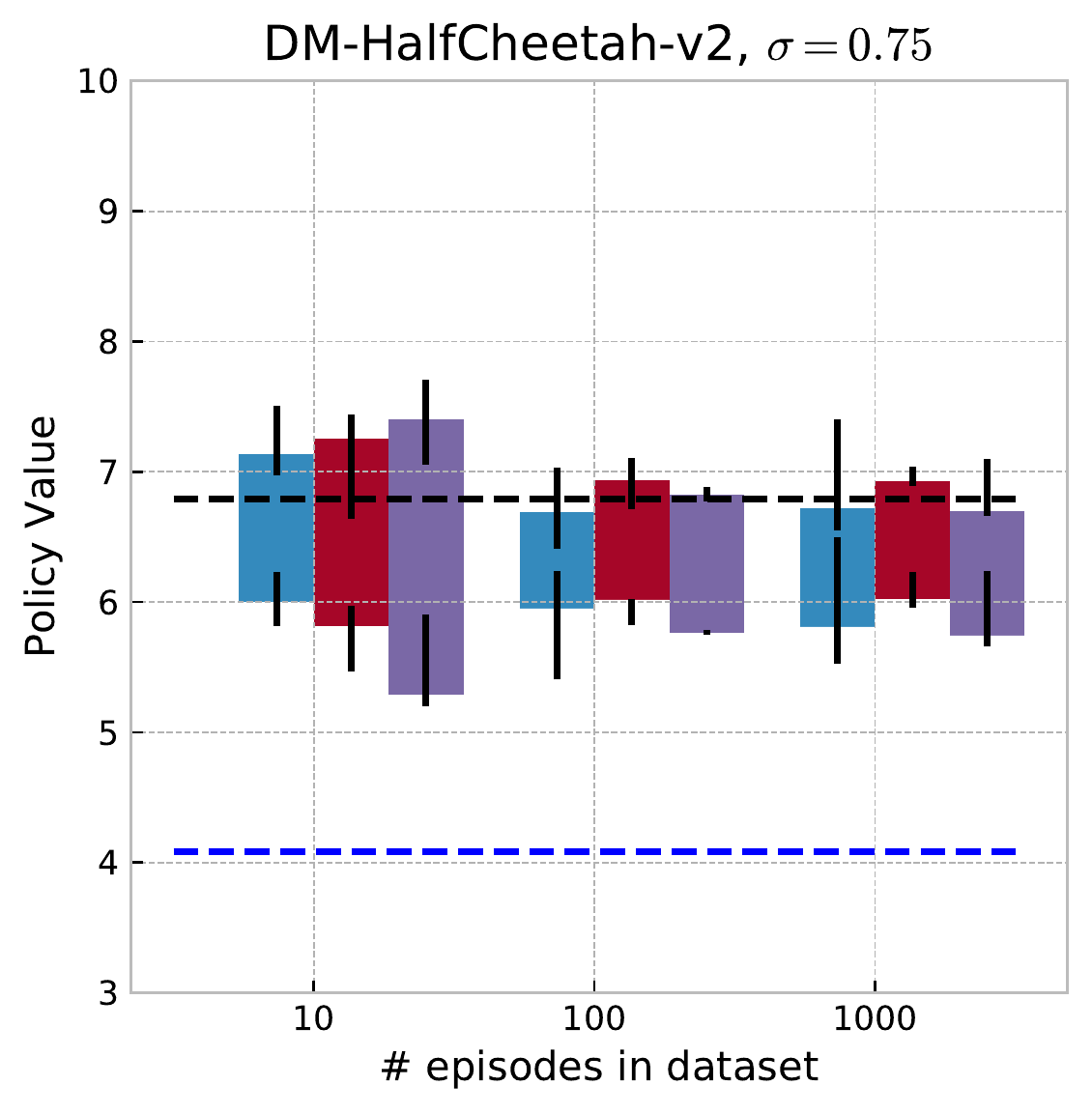}
\includegraphics[scale
=0.375]{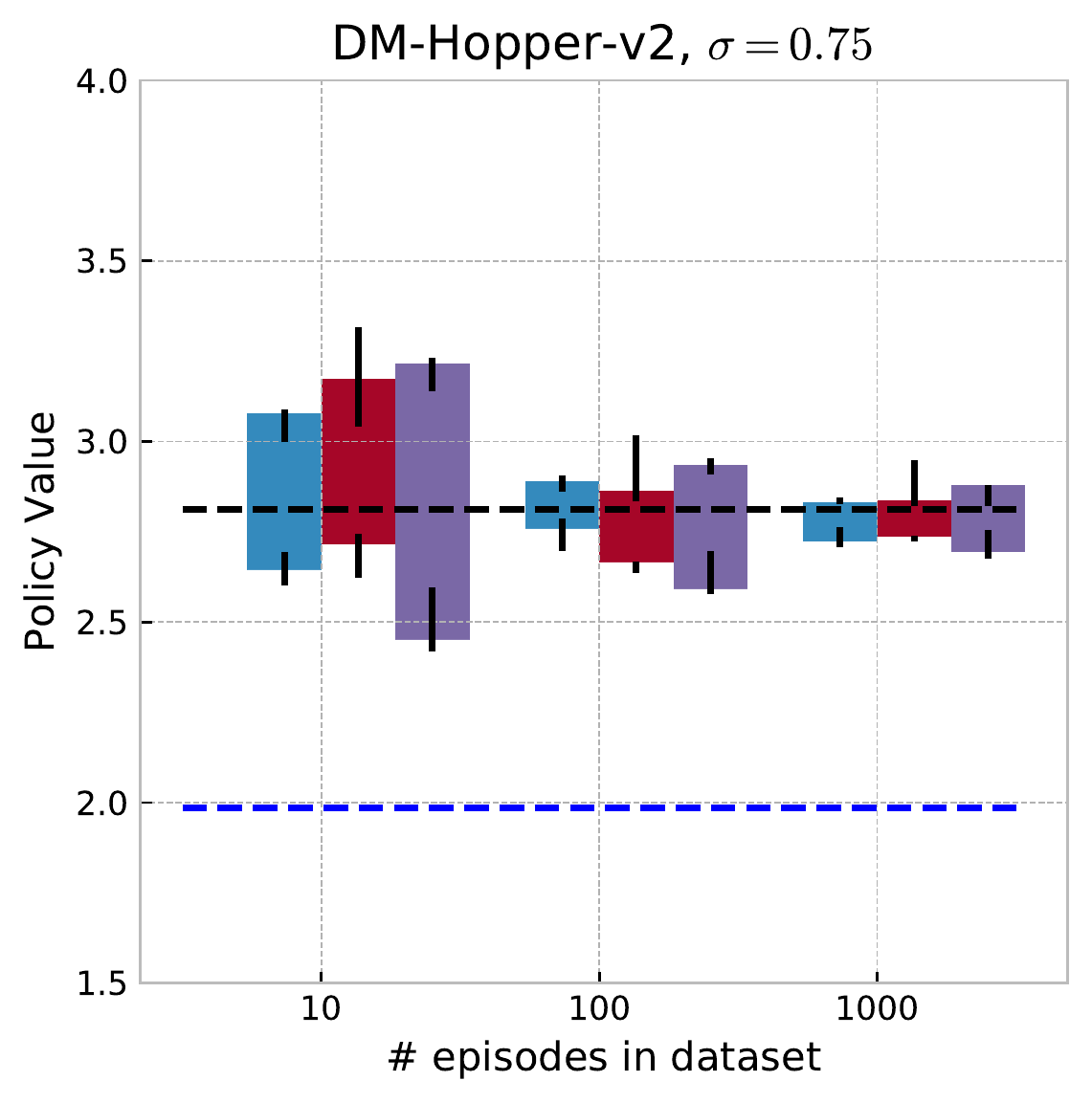}
\end{center}
 \caption{Additional results for MB. We plot different the confidence interval for different values of the noise scale.}
 \label{app:mujoco}
 \end{figure}
 
 \subsection{Experimental Details}
 
 For the ease of reproducibility we provide details of our experimental setup.  For all methods we normalize the states and rewards to have mean of $0$ and standard deviation of $1$. We normalize the terminating rewards accordingly. For all neural networks we use orthogonal initialization.
 
 \paragraph{Fitted Q-Evaluation}  We use 2 layer MLP with 256 hidden units and perform standard TD-$0$ policy evaluation. In order to compute the target value for FQE, we use target networks that are updated using Polyak averaging with $\tau=0.005$ as in \cite{lillicrap2015continuous}. For the results we plot predictions from the target network.
\paragraph{Model based policy evaluation} We perform model based policy evaluation as described in \cite{hanna2016high}. We found that to make the algorithm stable in low data regime, it is crucial to apply L2 regularization and clip states and rewards generated by the models to the limits observed in the training data. The forward model predicts offset from the current state: $f_\theta(s, a) \mapsto s' - s$ and trained by optimizing a mean squared error $\cfrac{1}{N}\sum_{i=1}^N(f_\theta(s_i, a_i) + (s_i - s'_i))^2$, while for rewards we train a model that regresses rewards directly $\cfrac{1}{N}\sum_{i=1}^N(g_\theta(s_i, a_i) - r_i)^2$. We also train a model that predicts terminating condition via binary classification.

\end{document}